\documentclass[10pt,a4paper]{article}

\usepackage[margin=1in]{geometry}
\usepackage{microtype}
\usepackage{booktabs}
\usepackage{graphicx}
\usepackage{bm}

\usepackage{amsmath,amssymb,amsthm,mathtools}
\newtheorem{theorem}{Theorem}
\newtheorem{proposition}[theorem]{Proposition}
\newtheorem{corollary}[theorem]{Corollary}
\newtheorem{lemma}[theorem]{Lemma}
\theoremstyle{definition}
\newtheorem{assumption}{Assumption}
\newtheorem{definition}{Definition}
\newtheorem{remark}{Remark}
\newtheorem{example}{Example}

\usepackage[colorlinks=true,linkcolor=blue,citecolor=blue,urlcolor=blue]{hyperref}

\newcommand{\U}{\mathcal{U}}
\newcommand{\W}{\mathcal{W}}
\newcommand{\Lset}{\mathcal{L}}
\newcommand{\Ac}{A_{\mathrm{cert}}}
\newcommand{\Prob}{\mathbb{P}}

\newcommand{\R}{\mathbb{R}}
\newcommand{\Ra}{\operatorname{Range}}
\newcommand{\tr}{\operatorname{tr}}
\newcommand{\lorc}{\lambda_{\mathrm{or}}}
\newcommand{\lfix}{\lambda_{\mathrm{fix}}}
\newcommand{\lD}{\lambda_{D}}

\usepackage{authblk}

\title{\LARGE\bf Observation Design for Certified Control Authority:\\ Projection--Estimability Separation and Active-Face Equivalence}

\author[1]{Guangxi Wan}
\author[1,2]{Hualong Du}
\author[1]{Yuqi Liu}
\author[1]{Qingwei Dong}
\author[1]{Qingxin Li}
\author[1]{Hongfei Bai}
\author[1,*]{Peng Zeng}

\affil[1]{State Key Laboratory of Robotics and Intelligent Systems, Shenyang Institute of Automation, Chinese Academy of Sciences, Shenyang 110016, China}
\affil[2]{University of Chinese Academy of Sciences, Beijing 100049, China}
\affil[*]{Corresponding author: \texttt{zp@sia.cn}}

\date{}

\begin{document}

\maketitle

\begin{abstract}
A sound runtime admission gate executes only actions it can certify, and certifies only what its observations support. This paper asks how observations should be designed to maximize the set of actions that can be safely admitted, and shows the question is not a re-vocabulary of classical design problems. First, a \emph{projection--estimability separation}: decomposing a constraint normal as $c=c_{\Ra}+c_{\ker}$ relative to an information matrix, two-point discrimination along $c$ becomes arbitrarily reliable as the budget grows whenever $c_{\Ra}\neq0$, while robust admission of an action with normal $c$ is impossible at \emph{every} budget whenever $c_{\ker}\neq0$; such mixed directions are generic at any deficient rank, and at full rank the decoupling is bounded by the Kantorovich ratio and diverges with the condition number. Discrimination-optimal designs, being corner solutions of a linear criterion, land in exactly this regime. Second, an \emph{active-face equivalence theorem}: the permissiveness-optimal design is $L$-optimal for a target matrix generated endogenously by the action faces that become certification bottlenecks, weighted inversely by their remaining slack; single-face collapse recovers $c$-optimal and goal-oriented design exactly, and a Carath\'eodory argument yields a bottleneck certificate of at most $r(r+1)/2+1$ faces. Around these we assemble exact certifiability per convex contract mode (whence $\kappa\approx3.29$ is derived, not calibrated), the $\sqrt r$ price of contract-agnostic design, an information-to-slack transfer theorem with a curvature-budget corollary, and a two-part audit in which a design meeting every margin requirement still leaves an action face at a certification cost above ten times its testing cost, in every probe library tested.
\\[0.8em]
\end{abstract}

\vspace{1.5em}

\section{Introduction}

A companion paper \cite{companion} establishes when a runtime assurance architecture guarantees safety independently of its upstream generator: the verifier must certify a \emph{set} of admissible actions containing no nonviable action, and under partial observation any admission mechanism obeys $\bar\alpha+\beta+\delta\ge1$ whenever two states at total-variation distance $\delta$ demand different decisions. That theory fixes what the gate must output. It leaves the resource question on which deployment turns: observations are designed, and the design determines how large the certified action set can be.

The classical instruments optimize something else. Distinguishability theory in inverse problems \cite{isaacson1986,gisser1990,cheney1992} optimizes the ability to tell two states apart; model-discrimination design \cite{atkinson1975,atkinson1975seq} optimizes power against a rival model; goal-oriented design \cite{attia2018,wu2023,zhong2026} minimizes posterior uncertainty of a downstream quantity; parameter-oriented design \cite{kiefer1960,elfving1952,fedorov1972,pukelsheim1993} maximizes estimation information. None has as its object a \emph{certified action set with hard admission semantics}. The question of this paper is whether that object changes the mathematics or only the vocabulary. We give two theorems saying that it changes the mathematics, and we are explicit about which classical results they do \emph{not} displace.

\paragraph{Wall one: admission is not discrimination.}
Theorem~\ref{thm:separation} isolates a structural feasibility break. Two-point discrimination consumes the \emph{projection} of a state difference onto the measured subspace: it is possible as soon as $c^{\!\top}Mc>0$ and becomes perfect as the budget grows. Robust admission consumes the \emph{estimability} of the constraint functional: it requires $c\in\Ra M$. The gap between these---directions with a nonzero projection and a nonzero null component---is where a design can make discrimination perfect and admission impossible simultaneously, at every budget. Discrimination criteria are linear in the design measure and therefore attain their optima at corners, which is precisely how a design lands in the gap.

\paragraph{Wall two: multi-face admission is not one more $c$-optimal problem.}
For a single active constraint face, permissiveness-optimal design coincides \emph{exactly} with $c$-optimal design and hence with goal-oriented design for that functional; we state this rather than obscure it, and claim nothing there. Theorem~\ref{thm:faceequiv} treats the case that admission actually presents: several faces, coupled because certifying an action requires all of its faces and certifying a set requires the worst action. The optimum is characterized by a general-equivalence condition in which the target matrix
\[
A^\star=\sum_{\ell\in\mathcal I^\star}\frac{\mu_\ell}{r^\star_\ell}\,c_\ell c_\ell^{\!\top}
\;\propto\;\sum_{\ell\in\mathcal I^\star}\mu_\ell\,\frac{c_\ell c_\ell^{\!\top}}{s_\ell-\tau^\star}
\]
is generated by the optimizer itself from the faces that become bottlenecks, amplified inversely by their remaining slack. General equivalence theory for minimax criteria is mature \cite{wong1992,sagnol2011}; the novelty here is the identity of the least-favourable object---active control constraints---and the design interpretation that follows: the experiment reallocates sensing budget toward the normals of the constraints currently limiting control authority.

\paragraph{Contributions.} Theorem~\ref{thm:fixed} (exact certifiability of a convex contract mode; $\kappa=3.29$ derived) and Proposition~\ref{prop:modes} (mode composition, multiplicity charged to usefulness) are the calibration foundation and carry no novelty claim. Theorem~\ref{thm:separation} and Theorem~\ref{thm:faceequiv} are the two walls. Theorem~\ref{thm:universal} prices contract-agnostic design at $\sqrt r$, tight, and is a Kiefer--Wolfowitz corollary. Theorem~\ref{thm:transfer} converts information into certified slack, exactly for constraints affine over the confidence region, with a curvature-budget corollary and an explicit counterexample against reading its divergence criterion locally. Section~\ref{sec:audit} audits both objectives on a semi-physical instance and locates, empirically, where they part.

\paragraph{Position.} Table~\ref{tab:position} states the position; its last column is discharged by Theorems~\ref{thm:separation} and~\ref{thm:faceequiv}, not asserted. Two conventions are protocol rather than theory and are labelled so: the decision metric $D$ (Section~\ref{sec:setting}) removes a scaling non-identifiability---margins scale with $D$, so a $D$ back-computed from the sensitivity under audit makes certifiability claims gameable---and carries no mathematical novelty.

\begin{table}[t]
\centering\small
\caption{Position. The last column is discharged in Sections~\ref{sec:sep} and~\ref{sec:permiss}.}
\label{tab:position}
\vspace{0.4em}
\begin{tabular}{p{3.2cm}p{2.9cm}p{3.5cm}p{2.4cm}p{2.4cm}}
\toprule
Literature & Object & Design objective & Output & Hard admission guarantee\\
\midrule
Distinguishability (EIT) \cite{isaacson1986,gisser1990,cheney1992} & two states & maximize distinguishability & detectability & no\\
$T$-optimal discrimination \cite{atkinson1975,atkinson1975seq} & rival models & maximize discriminating power & model choice & no\\
Goal-oriented OED \cite{attia2018,wu2023,zhong2026} & downstream QoI & minimize QoI uncertainty & prediction uncertainty & no\\
$D/G/c/L$-optimal OED \cite{kiefer1960,elfving1952,fedorov1972,pukelsheim1993,sagnol2011} & parameters, contrasts & estimation information & covariance & no\\
\textbf{This paper} & contract $+$ action set & information sufficient for \emph{admission} & certified action set & yes; $(\alpha,\beta)$ interface of \cite{companion}\\
\bottomrule
\end{tabular}
\end{table}

\section{Certification Geometry}\label{sec:setting}

Sensing is linearized at a nominal state: $y_s=J(u_s)x+\epsilon_s$, $\epsilon_s\sim\mathcal N(0,\Sigma)$ i.i.d., $u_s\in\U$ compact, budget $T$. Write $S(u):=\Sigma^{-1/2}J(u)D$, $Q(u):=S(u)^{\!\top}S(u)$, and for a design measure $\xi\in\mathcal P(\U)$, $M(\xi):=\int Q\,\mathrm d\xi$. The decision metric $D$ is anchored to documented decision thresholds per the convention of \cite{companion}. A contract $\varphi$ is \emph{$D$-stable} if constant on $D$-balls of radius $\tfrac12$; industrial contracts are intersections of scalar constraints, so the unsafe set is a union of \emph{violation modes} $\mathcal U_\varphi=\bigcup_{k\le K}\mathcal U_k$.

\begin{definition}[Mode margins]\label{def:margin}
$\W_k:=\{D^{-1}(x_s-x_u):x_s\in\mathcal S_m,\,x_u\in\mathcal U_k\}$ and $m_{T,k}(\xi):=(T\inf_{w\in\W_k}w^{\!\top}M(\xi)w)^{1/2}$.
\end{definition}

\begin{assumption}[Mode regularity]\label{ass:reg}
$\U$ compact, $u\mapsto Q(u)$ continuous; per mode, the whitened images of $\mathcal S_m$ and $\mathcal U_k$ are closed convex at positive distance with attained nearest pair. Structurally blind directions are quotiented out; $r$ is the recoverable quotient dimension.
\end{assumption}

Convexity is assumed per mode, where it is realistic (one constraint face), not for the union, where it is not.

The following functional carries the estimability convention that the rest of the paper depends on, and is stated separately because the naive pseudoinverse expression does \emph{not} encode it.

\begin{definition}[Extended estimability functional]\label{def:rc}
For $M\succeq0$ and $c\in\R^{n}$,
\[
r_c(M):=\begin{cases}\sqrt{c^{\!\top}M^{\dagger}c}, & c\in\Ra M,\\[2pt] +\infty, & c\notin\Ra M.\end{cases}
\]
\end{definition}

\begin{remark}[Why the convention is not cosmetic]\label{rem:pinv}
The Moore--Penrose quadratic form is finite off the range: for $M=\operatorname{diag}(1,0)$ and $c=(\cos30^\circ,\sin30^\circ)$ it equals $0.75$. What diverges is the support function $\sup\{c^{\!\top}w:\,Tw^{\!\top}Mw\le q\}$, since $w$ may run to infinity in $\ker M$ at zero cost. All robustification penalties below are support functions, so they must be written with $r_c$, never with $\sqrt{c^{\!\top}M^{\dagger}c}$.
\end{remark}

\section{Exact Certifiability: Modes and Their Composition}\label{sec:exact}

\begin{theorem}[Exact certifiability of a convex mode]\label{thm:fixed}
Under Assumption~\ref{ass:reg} and a fixed design $\xi$, mode $k$ is $(\alpha,\beta)$-certifiable iff $m_{T,k}(\xi)\ge z_{1-\alpha}+z_{1-\beta}$.
\end{theorem}

\begin{proof}
Sufficiency: the nearest-pair mid-normal hyperplane test; convexity places each whitened image at signed distance $\ge$ its share of $m_{T,k}$, giving uniform errors $(\Phi(-(m_{T,k}-c)),\Phi(-c))$, and $c=z_{1-\beta}$ is feasible iff the margin condition holds. Necessity: restricting to the nearest pair gives a simple Gaussian shift test, whose maximal power at level $\alpha$ is $\Phi(m_{T,k}-z_{1-\alpha})$ by Neyman--Pearson. Unattained infima: $\varepsilon$-pairs.
\end{proof}

\begin{corollary}\label{cor:kappa}
At $\alpha=\beta=0.05$, $\kappa=2z_{0.95}=3.29$: the threshold constant is a function of the declared levels, not a tuning parameter.
\end{corollary}

\begin{example}[Two-sided violation]\label{ex:twosided}
$\mathcal S=\{0\}$, $\mathcal U=\{\pm m\}$, $Y\sim\mathcal N(x,1)$: the nearest-pair distance is $m$ per mode, but a verifier useful at $0$ must accept an interval, and with $|Y|\le z_{1-\beta/2}$ false admission at $\pm m$ is $\le\alpha$ iff $m\ge z_{1-\alpha}+z_{1-\beta/2}$---strictly above the single-mode threshold. Minimum pair distance does not characterize composite nonconvex testing.
\end{example}

\begin{proposition}[Mode composition]\label{prop:modes}
With $\psi=\bigwedge_k\psi_k$ (admit only if every mode test admits), $\psi$ is $(\alpha,\beta)$-certifying for $\varphi$ whenever $m_{T,k}(\xi)\ge z_{1-\alpha}+z_{1-\beta/K}$ for every $k$. The multiplicity is charged to usefulness, not reliability: for $x\in\mathcal U_k$, $\Prob_x(\psi=1)\le\Prob_x(\psi_k=1)\le\alpha$ with no division of $\alpha$; for $x\in\mathcal S_m$, $\Prob_x(\psi=0)\le\sum_k\Prob_x(\psi_k=0)$. Per-mode necessity persists by restriction, bracketing the contract requirement within a usefulness-side gap $z_{1-\beta/K}-z_{1-\beta}$.
\end{proposition}

\begin{remark}[Correction, and why the error hides]\label{rem:direction}
An earlier draft allocated the multiplicity to reliability ($z_{1-\alpha/K}+z_{1-\beta}$). This is wrong---the conjunction protects each unsafe mode by its own full-level test, and it is the safe state that faces $K$ chances of rejection---but at $\alpha=\beta$ the two formulas coincide numerically (both $3.6048$ at $K=2$, $5\%$), which is why the error survives casual checking; at $\alpha=0.01$, $\beta=0.05$, $K=2$ they differ. The natural industrial certificate is therefore a \emph{margin vector} $(m_{T,1},\dots,m_{T,K})$, one entry per violation face, not a scalar for the union.
\end{remark}

\begin{theorem}[Policy-free converse]\label{thm:oracle}
With $\sigma(w):=\sup_u\|S(u)w\|$: for any adaptive policy the trajectory laws of $x_u$ and $x_u+Dw$ satisfy $d_{\mathrm{TV}}\le\tfrac{\sqrt T}{2}\sigma(w)$; hence no policy certifies mode $k$ with $1-\beta>\alpha+\tfrac{\sqrt T}{2}\inf_{\W_k}\sigma(w)$, and $\sigma(w)=0$ makes the laws coincide at every budget.
\end{theorem}

\section{Wall One: Projection--Estimability Separation}\label{sec:sep}

\begin{theorem}[Projection--estimability separation]\label{thm:separation}
Fix a design with information matrix $M\succeq0$, write $\Ra:=\Ra M$, $\ker:=\ker M$, and decompose a constraint normal $c=c_{\Ra}+c_{\ker}$. Consider (A) two-point discrimination between $x_0$ and $x_0+mDc$, and (B) robust certification of the functional $c^{\!\top}w$ over the confidence set $C_T=\{w:Tw^{\!\top}Mw\le q\}$. Then
\[
\text{(A) is possible}\iff c_{\Ra}\neq0,\qquad\qquad
\text{(B) is possible}\iff c_{\ker}=0,
\]
and the three regimes are
\[
\begin{array}{lll}
c_{\Ra}=0 & \text{discrimination impossible} & \text{admission impossible}\\
c_{\ker}=0 & \text{discrimination possible} & \text{admission possible}\\
c_{\Ra}\neq0,\ c_{\ker}\neq0 & \textbf{discrimination possible} & \textbf{admission impossible}
\end{array}
\]
In the third regime the separation is total: $d_{\mathrm{TV}}\to1$ as $T\to\infty$, while $\sup_{C_T}c^{\!\top}w=+\infty$ at every $T$. If $0<\operatorname{rank}M<n$, the third regime is generic---its complement is the union of two proper subspaces, hence Lebesgue-null.
\end{theorem}

\begin{proof}
The two-point whitened separation is $Tm^2c^{\!\top}Mc=Tm^2c_{\Ra}^{\!\top}Mc_{\Ra}$, positive iff $c_{\Ra}\neq0$, and the Gaussian two-point total variation increases to $1$ with $T$. The support function of $C_T$ is $\sqrt{q/T}\,r_c(M)$ by Definition~\ref{def:rc}: finite iff $c\in\Ra M=(\ker M)^{\perp}$, i.e.\ $c_{\ker}=0$; otherwise $w=t\,c_{\ker}$ has $w^{\!\top}Mw=0$ and $c^{\!\top}w=t\|c_{\ker}\|^2\to\infty$. Genericity: $\{c:c_{\Ra}=0\}\cup\{c:c_{\ker}=0\}=\ker M\cup\Ra M$, both proper subspaces when the rank is strictly between $0$ and $n$.
\end{proof}

Perfect knowledge of \emph{which of two worlds} one is in does not imply enough information to certify an action \emph{over the unresolved state set}. The mechanism is that the two problems consume different features of the same design: a projection and an estimability.

\begin{corollary}[Implicit spanning constraint]\label{cor:span}
The admission design problem of Section~\ref{sec:permiss} is feasible only on designs with $\operatorname{span}\{c_\ell\}\subseteq\Ra M(\xi)$ over the active faces; the discrimination problem imposes no analogous constraint. This is the structural reason a discrimination-optimal design can be admission-infeasible.
\end{corollary}

\begin{example}[Three designs, all different]\label{ex:2d}
In $\R^2$ with $Q_1=e_1e_1^{\!\top}$, $Q_2=e_2e_2^{\!\top}$ and $c=(\cos30^\circ,\sin30^\circ)$: the discrimination objective $c^{\!\top}M(p)c=p\,c_1^2+(1-p)c_2^2$ is linear in the weight $p$ on $Q_1$ and is maximized at the corner $p=1$; at $p=1$, $c\notin\Ra M$ and $r_c(M)=+\infty$, so no action on this face is ever certifiable, although two-point discrimination along $c$ remains solvable at every budget since $|e_1^{\!\top}c|=0.866$. The permissiveness objective $r_c(M(p))^2=c_1^2/p+c_2^2/(1-p)$ is minimized at the interior Elfving point $p^\star=c_1/(c_1+c_2)=0.634$ with value $(c_1+c_2)^2=1.866$; the $D$-optimal design is $p=\tfrac12$. Linearity of the discrimination criterion in $\xi$ is what places its optimum at a corner, and corners are where ranks drop.
\end{example}

\begin{remark}[The full-rank continuum, and what the walls do not separate]\label{rem:kantorovich}
Rank deficiency is the limit of a continuous phenomenon, not an isolated pathology. For $M\succ0$ and unit $c$ define the decoupling ratio $\chi(c,M):=(c^{\!\top}Mc)(c^{\!\top}M^{-1}c)$, the product of the discrimination quality and the squared admission penalty along $c$. Then $\chi\ge1$ by Cauchy--Schwarz, with equality iff $c$ is an eigenvector of $M$, and by the Kantorovich inequality $\chi\le(1+\varkappa)^2/(4\varkappa)$ with $\varkappa$ the condition number---so the achievable decoupling grows like $\varkappa/4$ and diverges as the design becomes ill-conditioned (numerically: $\chi_{\max}=3.03$ at $\varkappa=10$, $25.5$ at $10^2$, $250.5$ at $10^3$, matching the bound to three digits). Two limits of this wall should be stated plainly. It separates admission from \emph{discrimination and simple-hypothesis power} criteria, which can remain optimal under a rank-deficient design; it does \emph{not} separate admission from single-functional variance criteria---$c$-optimal and single-QoI goal-oriented objectives detect this failure precisely, and in fact coincide with single-face admission design (Theorem~\ref{thm:faceequiv}(i)). The distinction from those criteria begins only when several action faces are coupled by hard certified-slack requirements, which is Wall Two.
\end{remark}

\section{Wall Two: Permissiveness Design and Active-Face Equivalence}\label{sec:permiss}

Index action--face pairs by $\ell=(a,j)\in\Lset$, with nominal slack $s_\ell:=-g_j(\hat x,a)$ and normal $c_\ell:=D^{\!\top}\nabla_xg_j(\hat x,a)$. Write $\eta:=\sqrt{q_{1-\gamma}/T}$. By Theorem~\ref{thm:transfer} the certified slack of face $\ell$ under design $\xi$ is
\[
\varsigma_\ell(\xi):=s_\ell-\eta\,r_{c_\ell}(M(\xi)),
\]
the certified action set is $\Ac(\xi)=\{a:\varsigma_{(a,j)}(\xi)\ge0\ \forall j\}$, and the \emph{permissiveness design problem} is
\begin{equation}\label{eq:perm}
\tau^\star=\max_{\xi\in\mathcal P(\U)}\ \min_{\ell\in\Lset}\ \varsigma_\ell(\xi).
\end{equation}

\begin{lemma}[Convexity of the estimability functional]\label{lem:convex}
On $\{M\succ0\}$ the map $M\mapsto r_c(M)=\sqrt{c^{\!\top}M^{-1}c}$ is convex.
\end{lemma}

\begin{proof}
Write $q(M)=c^{\!\top}M^{-1}c$, $v=M^{-1}c$, and for a symmetric direction $H$ set $a=v^{\!\top}Hv$, $b=v^{\!\top}HM^{-1}Hv\ge0$. Then $Dq[H]=-a$, $D^2q[H,H]=2b$, and
\[
D^2r_c[H,H]=\tfrac{1}{2}q^{-1/2}D^2q-\tfrac14 q^{-3/2}(Dq)^2=\frac{b}{\sqrt q}-\frac{a^2}{4q^{3/2}} .
\]
Cauchy--Schwarz in the $M$-inner product, $a=\langle M^{1/2}v,M^{-1/2}Hv\rangle$, gives $a^2\le(v^{\!\top}Mv)(v^{\!\top}HM^{-1}Hv)=qb$, whence $D^2r_c[H,H]\ge\tfrac{3b}{4\sqrt q}\ge0$.
\end{proof}

Convexity of $r_c$ does not follow from convexity of $q$---the square root of a convex function need not be convex---which is why the lemma is proved rather than asserted.

\begin{theorem}[Active-face equivalence for certified control authority]\label{thm:faceequiv}
Assume $\Lset$ finite, $\U$ compact, and that \eqref{eq:perm} attains its optimum at $\xi^\star$ with $M^\star:=M(\xi^\star)\succ0$ on the recoverable quotient. Write $r^\star_\ell:=r_{c_\ell}(M^\star)$ and let $\mathcal I^\star:=\{\ell:\varsigma_\ell(\xi^\star)=\tau^\star\}$ be the \emph{active faces}. Then:

\emph{(i) Single-face collapse.} If $|\mathcal I^\star|=1$, \eqref{eq:perm} reduces to $\min_\xi r_{c}(M(\xi))$, which is exactly $c$-optimal design for the contrast $c$---Elfving's geometry \cite{elfving1952}---and coincides with goal-oriented design for the linear functional $c^{\!\top}x$ \cite{attia2018}. No novelty is claimed at this level.

\emph{(ii) Concavity.} Each $\varsigma_\ell$ is concave in $\xi$ (Lemma~\ref{lem:convex} and linearity of $\xi\mapsto M(\xi)$), so \eqref{eq:perm} is a concave maximin program.

\emph{(iii) Least-favourable face mixture and equivalence condition.} There exist weights $\mu_\ell\ge0$ supported on $\mathcal I^\star$ with $\sum\mu_\ell=1$ such that, with the \emph{endogenous target matrix}
\[
A^\star:=\sum_{\ell\in\mathcal I^\star}\frac{\mu_\ell}{r^\star_\ell}\,c_\ell c_\ell^{\!\top},
\]
the design $\xi^\star$ satisfies the generalized sensitivity condition
\begin{equation}\label{eq:equiv}
\tr\!\left[(M^\star)^{-1}A^\star(M^\star)^{-1}Q(u)\right]\ \le\ \tr\!\left[(M^\star)^{-1}A^\star\right]\qquad\text{for all }u\in\U,
\end{equation}
with equality at every $u$ in the support of $\xi^\star$. Condition \eqref{eq:equiv} is exactly the equivalence condition of $L$-optimal design for the target $A^\star$: \emph{the permissiveness-optimal design is $L$-optimal for a target matrix the optimizer generates itself.}

\emph{(iv) Balance law.} On active faces $r^\star_\ell=(s_\ell-\tau^\star)/\eta$, hence
\[
A^\star=\eta\sum_{\ell\in\mathcal I^\star}\mu_\ell\,\frac{c_\ell c_\ell^{\!\top}}{s_\ell-\tau^\star}:
\]
faces with less remaining slack are amplified in the endogenous target. The design reallocates sensing budget toward the normals of the constraints currently limiting control authority.

\emph{(v) Sparse bottleneck certificate.} $A^\star$ lies in the convex hull of the rank-one matrices $c_\ell c_\ell^{\!\top}/r^\star_\ell$, $\ell\in\mathcal I^\star$, inside the $\tfrac{r(r+1)}2$-dimensional space of symmetric matrices on the quotient; by Carath\'eodory at most $\tfrac{r(r+1)}2+1$ faces are needed to represent it. Optimality of a design therefore admits a certificate naming at most that many bottleneck faces, however large the action library.
\end{theorem}

\begin{proof}
(i) is the reduction of \eqref{eq:perm} to a single penalty plus the classical identification. (ii) is Lemma~\ref{lem:convex} composed with a linear map, negated, and a minimum of concave functions. (iii): by (ii) and compactness, first-order optimality of the concave maximin gives multipliers $\mu$ on the active set with $\sum_\ell\mu_\ell\,\mathrm D\varsigma_\ell[\xi^\star;\delta_u-\xi^\star]\le0$ for all $u$. Moving mass $\varepsilon$ to $u$ gives $\mathrm dM/\mathrm d\varepsilon=Q(u)-M^\star$ and
\[
\frac{\mathrm d\varsigma_\ell}{\mathrm d\varepsilon}=\frac{\eta}{2r^\star_\ell}\Bigl[c_\ell^{\!\top}(M^\star)^{-1}Q(u)(M^\star)^{-1}c_\ell-(r^\star_\ell)^2\Bigr],
\]
and summing against $\mu$ yields $\sum_\ell\frac{\mu_\ell}{r^\star_\ell}c_\ell^{\!\top}(M^\star)^{-1}Q(u)(M^\star)^{-1}c_\ell\le\sum_\ell\mu_\ell r^\star_\ell$, which is \eqref{eq:equiv} after writing both sides as traces against $A^\star$ (using $\sum_\ell\frac{\mu_\ell}{r^\star_\ell}c_\ell^{\!\top}(M^\star)^{-1}c_\ell=\sum_\ell\mu_\ell r^\star_\ell$). Equality on the support is complementary slackness. (iv) substitutes the active-face identity $s_\ell-\eta r^\star_\ell=\tau^\star$. (v) is Carath\'eodory in $\mathbb S^{r}$.
\end{proof}

\begin{remark}[What is and is not new here]\label{rem:notelfving}
General equivalence theorems for minimax and linear design criteria are classical and comprehensive \cite{kiefer1960,fedorov1972,pukelsheim1993,wong1992}, including second-order cone formulations for multiresponse $c$/$A$/$T$/$D$ problems \cite{sagnol2011}; we claim no new design-theoretic machinery and in particular do not claim a new Elfving theorem. What the admission object contributes is the identity of the least-favourable player---\emph{active control faces}, entering through hard certified-slack constraints---together with the balance law (iv), which has no counterpart when the criterion is a variance: nothing in a variance objective ``runs out of slack''.
\end{remark}

\begin{remark}[Numerical verification]\label{rem:numeric}
On a three-dimensional instance with five probes and four faces, solving \eqref{eq:perm} to $\tau^\star=0.1422168$ yields two active faces and three support probes; recovering $\mu=(0.3453,0.6547)$ gives sensitivity values $\tr[(M^\star)^{-1}A^\star(M^\star)^{-1}Q(u)]$ equal to the threshold $1.5000882$ on all three support probes to $10^{-8}$ and strictly below it on the two non-support probes ($1.462$ and $1.088$), and the balance-law residual vanishes to machine precision. The certificate uses two faces against the Carath\'eodory bound of seven.
\end{remark}

\begin{remark}[Certification schedule]\label{rem:schedule}
As the budget grows, actions enter $\Ac$ in order of their worst ratio $s_\ell/r_{c_\ell}(M)$, and the design reorders that schedule: choosing a design is choosing which control authority becomes available first. Monotonicity holds throughout---$M(\xi_1)\succeq M(\xi_2)$ implies $\varsigma_\ell(\xi_1)\ge\varsigma_\ell(\xi_2)$ for all $\ell$, hence $\Ac(\xi_1)\supseteq\Ac(\xi_2)$.
\end{remark}

\begin{remark}[Two metrics, and the shortcut between them]\label{rem:tworoles}
The margin objectives of Sections~\ref{sec:exact} and~\ref{sec:design} ($w^{\!\top}Mw$ on witnesses) and the permissiveness objective ($r_c(M)^2$ on normals) are dual-metric quantities coinciding only at eigenvectors, with gap $\chi\ge1$ quantified in Remark~\ref{rem:kantorovich}. Margins decide whether the \emph{contract} is testable; penalties decide which \emph{actions} are certifiable. Conflating them is, in our reading, the shortcut by which admission design has passed for discrimination design.
\end{remark}

We close the section by recording the degenerate case in which the two objectives of this paper collapse into one.

\begin{corollary}[One-dimensional Gaussian specialization]\label{cor:1dgauss}
Suppose the admissible-state family is one-dimensional in decision coordinates: $w\in\{-mv,+mv\}$ with unit $v$, $m\ge m_{\min}>0$, and admission requires resolving which sign holds for a face normal $c$ with $c^{\!\top}v\neq0$. Then the certification criterion of Theorem~\ref{thm:fixed} is monotone in the standardized separation $m_{\min}\sqrt{T\,v^{\!\top}M(\xi)\,v}$, and the optimal certifying test is the matched-filter likelihood-ratio test of its proof. For a library of equal-cost scalar probes, each contributing $Q(u_h)=hh^{\!\top}$, the information increment along $v$ is $\Delta I(h)=(h^{\!\top}v)^2$; hence a cardinality-constrained optimal design selects probes in descending order of $|h^{\!\top}v|$.
\end{corollary}

\begin{proof}
The witness set of Definition~\ref{def:margin} is the single difference $2mv$, so $m_{T}=2m\sqrt{T\,v^{\!\top}M v}$ and Theorem~\ref{thm:fixed} reduces certifiability to a threshold on the standardized separation; the necessity half of its proof is the Neyman--Pearson (matched-filter) test on the pair. Independent probes contribute additively, $v^{\!\top}\!\big(\sum_i h_ih_i^{\!\top}\big)v=\sum_i(h_i^{\!\top}v)^2$, and under equal cardinality cost sorting nonnegative additive increments is optimal.
\end{proof}

\begin{remark}[The degenerate row of the trichotomy]\label{rem:1ddegenerate}
When design-time reachability restricts the family to $\operatorname{span}\{v\}$, the recoverable quotient is one-dimensional and the restricted information is the scalar $v^{\!\top}Mv$: the decomposition of Theorem~\ref{thm:separation} has no mixed direction, and certified admission and two-point discrimination are governed by the same quantity. The corollary therefore \emph{delimits} the separation rather than weakening it---it identifies exactly the degenerate geometry in which the separation collapses, and certificate-optimal sensing coincides with standardized-distinguishability-optimal sensing. A runtime instantiation of this boundary case is developed in a companion manuscript in preparation \cite{registry2026}.
\end{remark}

\begin{remark}[Design-time versus realized authority]\label{rem:designtimeauthority}
The certified control authority characterized throughout this paper is a \emph{design-time} object: the certified slack $\varsigma_\ell(\xi)=s_\ell-\eta\,r_{c_\ell}(M(\xi))$ is determined by the information matrix $M(\xi)$ and the constraint geometry at the nominal state, \emph{before} any measurement is realized, and it is monotone in information---enlarging $M$ in the Loewner order can only shrink the confidence region and hence can only enlarge the certified set (Theorem~\ref{thm:transfer}, Remark~\ref{rem:schedule}). A runtime system that must re-certify against the measurements it has actually collected evaluates a different functional: the expectation, over realizations, of the value certifiable once the region has been recentered at the realized estimate. That functional does not inherit the monotonicity. A realization moves the location of the region---and with it the nominal slacks $s_\ell=-g_j(\hat x,a)$---as well as its shape, while admission is a threshold on $\varsigma_\ell$ rather than a smooth functional of it, so certified value is not concave in the realized estimate and expectations of it need not increase with information. The two functionals agree in neither value nor optimizer in general, and design-time statements such as Theorem~\ref{thm:fixed} should not be read as claims about the realized quantity.
\end{remark}

\section{The Price of Universality}\label{sec:design}

\begin{theorem}[Universal-design approximation; a Kiefer--Wolfowitz corollary]\label{thm:universal}
On the $r$-dimensional recoverable quotient the $D$-optimal design satisfies $Q(u)\preceq r\,M(\xi_D)$ for all $u$; hence $w^{\!\top}M(\xi_D)w\ge\sigma(w)^2/r$ for all $w$ and, dually, $r_c(M(\xi_D))^2\le r\inf_\xi r_c(M(\xi))^2$ for every $c$ in the quotient. The universal design loses at most $\sqrt r$ of margin and at most $\sqrt r$ of certified-slack penalty; the factor is worst-case tight (coordinate probes, uniform design). The matrix inequality is the equivalence theorem of \cite{kiefer1960} in multiresponse form \cite{fedorov1972}.
\end{theorem}

\begin{proof}
First-order optimality of $\log\det$ gives $\tr(M(\xi_D)^{-1}Q(u))\le r$, hence $\lambda_{\max}(M^{-1/2}QM^{-1/2})\le r$. The primal statement follows by quadratic forms and $\sup_u$. For the dual, $Q(u)\preceq rM_D$ for all $u$ implies $M(\xi)\preceq rM_D$ for every $\xi$, hence $M_D^{-1}\preceq r\,M(\xi)^{-1}$ on the quotient; take quadratic forms at $c$ and the infimum over $\xi$. Tightness by the coordinate construction.
\end{proof}

\begin{corollary}\label{cor:price}
Certifying a contract whose witnesses span $r$ competing directions with the universal design requires $T\ge r(\kappa/\inf\sigma)^2$; the design-diversity factor $r$ is the price of not knowing the contract, and by tightness is a genuine worst case rather than a proof artifact.
\end{corollary}

\section{From Information to Certified Slack}\label{sec:transfer}

\begin{theorem}[Transfer; exact for constraints affine over the confidence region]\label{thm:transfer}
Let $C=\{\hat x+Dw:Tw^{\!\top}Mw\le q\}$, intersected where needed with the design-time reachable set, and let $g_j(\hat x+Dw,a)=g_j(\hat x,a)+c_j(a)^{\!\top}w$ hold \emph{for all points of $C$}. Then $a$ is robustly feasible on face $j$ over $C$ iff
\[
g_j(\hat x,a)+\sqrt{q/T}\;r_{c_j(a)}(M)\le0 ,
\]
the left side being $+\infty$ exactly when $c_j\notin\Ra M$ and the null direction is not excluded by reachability. Moreover $M_1\succeq M_2\Rightarrow C(M_1)\subseteq C(M_2)\Rightarrow\Ac(M_1)\supseteq\Ac(M_2)$.
\end{theorem}

The proof is the ellipsoid support function, classical in the set-membership lineage \cite{bertsekas1971,kurzhanski1997}; the content is the accounting, and the scope---affine \emph{over the region}, not first-order affine at its center---is load-bearing.

\begin{example}[Null-space curvature]\label{ex:curvature}
$M=\operatorname{diag}(1,0)$, $g(w)=w_2^2-\varepsilon$: at $w=0$ the gradient is $c=0\in\Ra M$ and the first-order penalty vanishes, yet $C$ is unbounded in $w_2$ and $\sup_Cg=+\infty$. For nonlinear constraints, $c\in\Ra M$ does not imply a finite penalty; curvature re-enters through the null space, and the divergence criterion is claimed for affine constraints only.
\end{example}

\begin{corollary}[Curvature budget]\label{cor:curvature}
If $C$ is bounded (null space excluded by design-time reachability) and $\|M^{\dagger/2}\nabla^2_wg_j\,M^{\dagger/2}\|_{\mathrm{op}}\le L_j$ on $C$, then $g_j(\hat x,a)+\sqrt{q/T}\,r_{c_j}(M)+\tfrac{L_jq}{2T}\le0$ is sufficient for robust feasibility, and the same expression with $-\tfrac{L_jq}{2T}$ and strict positivity is sufficient for robust infeasibility. The linear penalty is $O(T^{-1/2})$ and the curvature budget $O(T^{-1})$, so the affine formula is recovered asymptotically with the sandwich quantifying the pre-asymptotic regime.
\end{corollary}

\section{Saturation Classes of the Certifiable Dimension}\label{sec:saturation}

With $N(T,\xi):=\#\{i:T\lambda_i(M(\xi))\ge\kappa^2\}$ an integer staircase, slope statements refer to the \emph{fitted slope} over the window $1<N<r$.

\begin{proposition}[Three response classes]\label{prop:sat}
(i) \emph{finite cap}: norm convergence $M_m\to M_\infty$ caps $N$ independently of $m$; (ii) \emph{zero cap}: annihilation of the decision subspace gives $N\equiv0$; (iii) \emph{unsaturated within budget}: channels not asymptotically dominated grow $N$ until a structural bound---for a steady-state diffusive network with $W$ ports, reciprocity bounds the accumulated information rank, hence $N^{*}\le W(W-1)$.
\end{proposition}

\begin{corollary}[Fitted slope]\label{cor:slope}
For a geometrically decaying whitened spectrum $\lambda_n\propto\rho^{2n}$ with two directions per mode, the fitted slope of $N$ against $\ln D$ is $2/\ln(1/\rho)$.
\end{corollary}

Numerically (flat 32-point scan): $\rho=0.4175$ gives $2.19$ vs.\ prediction $2.29$; $0.55$: $3.29$ vs.\ $3.35$; $0.70$: $5.42$ vs.\ $5.61$. Libraries amplifying high modes raise the slope (mixed library at $\rho=0.4175$: $2.99$), so the prediction fixes a value and a deviation direction. This is a physical corollary and a falsifiable prediction, not a central novelty: mode counting against a noise threshold is close to classical resolution analysis \cite{isaacson1991}.

\section{Two Audits: Margins and Certified Slack}\label{sec:audit}

Theorems~\ref{thm:separation} and~\ref{thm:faceequiv} predict that designing for testability and designing for admissibility are different problems. This section audits both on a semi-physical instance and reports where they agree, where they part, and by how much.

\paragraph{Instance.} The channel is a thirteen-mode interface model---a uniform mode and cosine/sine pairs for harmonics $1$--$6$, transmitted with geometric modal attenuation $\rho^{n}$---observed through an excitation library of angular scans, curvature and slope sensing, a low-noise mean-gap probe, a restricted-sector scan and a band-pass probe, each deployment costing one unit of budget. The action layer comprises six operating decisions (cold-start ramp, hot restart, overspeed, single-plane balance, extended service, acceptance without re-shim), each carrying affine constraint faces drawn from seven contract functionals: clearance evaluated at critical angular stations, total indicated runout, first-harmonic asymmetry, mean gap, an angular-gradient seizure proxy, and a high-harmonic waviness index. \emph{This instance is constructed for the present study; face normals and nominal slacks are illustrative and are not taken from any qualification record.} Confidence sets are at $\gamma=0.05$, so $q=\chi^2_{0.95,13}=22.36$, and $\kappa=3.29$ throughout.

\subsection{Margin audit}

With witnesses drawn from contract normals (never the coordinate basis, which would reproduce the tightness construction of Theorem~\ref{thm:universal} by fiat), the realized ratio $\lorc/\lD$ lies between $1.00$ and $6.50$ against the worst-case bound $r=13$, and equals $1.15$ for the angular-scan library closest to conventional instrumentation---the universal design within $7\%$ of the directionwise oracle in margin. The value of knowing the contract, $\lfix/\lD$, reaches $5.28$ only for narrowband lock-in probes and collapses toward $1$ for broadband scans. The headroom $\lorc/\lfix$ above the best fixed design is at most $1.23$ in every contract-normal configuration, while a positive control with per-band witnesses and specialized probes raises it to $3.16$, confirming that the procedure detects large headroom when it exists. The equivalence identity $\max_u\tr(M(\xi_D)^{\dagger}Q(u))=r$ holds to six decimals throughout. Under a pre-registered rule ($<1.5$ demote, $\ge2.0$ promote) this demotes the adaptive-design question \emph{for margin objectives}.

\subsection{Certified-slack audit}

We now design for $\varsigma$ rather than for margins. Three findings.

\emph{(a) A testability-optimal design can be admission-infeasible on a mixed normal.} Let the contract be tested through runout, first-harmonic asymmetry, waviness and the gradient proxy, while the action to be certified is a hot restart whose binding face is clearance at a critical station---a dense functional with a mean-gap component. The margin-maximin design places its entire mass on slope sensing, which annihilates the uniform mode: the resulting $M$ has rank $12$, and the clearance normal decomposes with $\|c_{\Ra}\|=0.926$ and $\|c_{\ker}\|=0.378$. This is exactly the third regime of Theorem~\ref{thm:separation}. Testing that face costs $T_{\mathrm{test}}=8.3$ deployments; certifying an action on it is impossible at every budget, $r_c(M)=+\infty$. The permissiveness design moves $0.23\%$ of the mass onto a full angular scan, restoring rank $13$ and certifying the same face at $T_{\mathrm{cert}}=127$. The failure is a feasibility break, not a degradation, and it is repaired by a change in the design that is invisible to the margin objective.

\emph{(b) The break is common, but its dependence on contract size is a property of the library, not a law.} The seven contract functionals generate $127$ non-empty face subsets, which we enumerate exhaustively rather than sample. For a subset $F$ let $\xi^\star_{\mathrm{margin}}(F)$ be the margin-maximin design over the witnesses $\{c_\ell\}_{\ell\in F}$, $M^\star_F$ its information matrix, and
\[
\delta_{\mathrm{null}}(F):=\max_{\ell\in F}\frac{\|P_{\ker M^\star_F}c_\ell\|}{\|c_\ell\|},
\]
so that $\delta_{\mathrm{null}}(F)=0$ iff the margin-optimal design is admission-feasible for $F$. We also report the practically meaningful quantity, the budget ratio, which by Theorem~\ref{thm:transfer} and Remark~\ref{rem:kantorovich} is proportional to the decoupling ratio:
\[
\frac{T_{\mathrm{cert}}}{T_{\mathrm{test}}}=\frac{q}{\kappa^2s^2}\,\chi(c,M),
\]
here at nominal slack $s=0.8$.

On the original library the exact-failure rate $p_{\mathrm{fail}}(n_f)$ over all subsets is $(0.00,0.19,0.34,0.37,0.29,0.14,0.00)$ for $n_f=1,\dots,7$: an interior peak at $n_f=4$. Three observations forbid reading this shape as general.

First, the zero at $n_f=1$ is not implied by the theory---and cannot be, since Theorem~\ref{thm:separation} states precisely that $c^{\!\top}Mc>0$ does not entail $c\in\Ra M$. It is an alignment property of this library: applying a common random orthogonal rotation to all probes, which leaves every $Q(u)$'s rank and spectrum unchanged and alters only the incidence between probe subspaces and contract normals, raises $p_{\mathrm{fail}}(1)$ from $0.00$ to $0.43$, $0.86$ and $0.57$ under three seeds. Second, the peak location is not stable: across ten library perturbations it sits at $n_f=1$, $2$, $3$ or $4$, or the profile is monotone increasing with no interior peak at all (Table~\ref{tab:c2} and Fig.~\ref{fig:robust}). Third, the argument that many faces force the design to span is false in general: a rank-one $M=vv^{\!\top}$ gives $c_\ell^{\!\top}Mc_\ell>0$ for arbitrarily many normals provided $v^{\!\top}c_\ell\neq0$.

Exact-rank failure is moreover a fragile indicator. Adding $0.1\%$ of a full angular scan to every probe drives $p_{\mathrm{fail}}$ to zero at every $n_f$---while leaving the fraction of subsets with $T_{\mathrm{cert}}/T_{\mathrm{test}}>10$ unchanged at $21/35$ for $n_f=3$. Removing the slope probe likewise eliminates exact failure entirely and simultaneously \emph{raises} near-inestimability, to $31/35$ at $n_f=3$ and $21/21$ at $n_f=5$: the binary and the practical indicators can move in opposite directions. The certifiable conclusion is therefore about the mechanism and the budget ratio, not about rank or about contract size: \emph{a design satisfying every margin requirement leaves at least one face at a certification cost above ten times its testing cost in $60\%$ to $100\%$ of three-face contracts in every library we tested}, and what governs this is the incidence between the measured subspace and the contract normals, for which $n_f$ is only a coarse proxy.

\begin{table}[t]
\centering\small
\caption{Exact-failure rate $p_{\mathrm{fail}}(n_f)$ over all $127$ face subsets, across library perturbations. The peak location is library dependent and the profile need not be interior-peaked; the rotated control preserves every probe's rank and spectrum and changes only the alignment with the contract normals.}
\label{tab:c2}
\vspace{0.4em}
\begin{tabular}{lcccccccc}
\toprule
library & $n_f{=}1$ & $2$ & $3$ & $4$ & $5$ & $6$ & $7$ & peak\\
\midrule
mixed (original) & $0.00$ & $0.19$ & $0.34$ & $0.37$ & $0.29$ & $0.14$ & $0.00$ & $4$\\
full-rank enriched ($+0.1\%$ scan) & $0.00$ & $0.00$ & $0.00$ & $0.00$ & $0.00$ & $0.00$ & $0.00$ & ---\\
leave out slope sensing & $0.00$ & $0.00$ & $0.00$ & $0.00$ & $0.00$ & $0.00$ & $0.00$ & ---\\
leave out curvature & $0.00$ & $0.19$ & $0.34$ & $0.37$ & $0.29$ & $0.14$ & $0.00$ & $4$\\
leave out full scan & $0.43$ & $0.33$ & $0.37$ & $0.37$ & $0.29$ & $0.14$ & $0.00$ & $1$\\
leave out mean-gap probe & $0.00$ & $0.19$ & $0.34$ & $0.37$ & $0.29$ & $0.14$ & $0.00$ & $4$\\
leave out band-pass & $0.00$ & $0.19$ & $0.34$ & $0.37$ & $0.29$ & $0.14$ & $0.00$ & $4$\\
narrowband-heavy & $0.71$ & $0.95$ & $1.00$ & $1.00$ & $1.00$ & $1.00$ & $1.00$ & $3$\\
rotated control, seed 1 & $0.43$ & $0.48$ & $0.40$ & $0.31$ & $0.24$ & $0.14$ & $0.00$ & $2$\\
rotated control, seed 2 & $0.86$ & $0.95$ & $1.00$ & $1.00$ & $1.00$ & $1.00$ & $1.00$ & $3$\\
rotated control, seed 3 & $0.57$ & $0.48$ & $0.46$ & $0.40$ & $0.29$ & $0.14$ & $0.00$ & $1$\\
\bottomrule
\end{tabular}
\end{table}

\begin{figure}[t]
\centering
\includegraphics[width=0.62\textwidth]{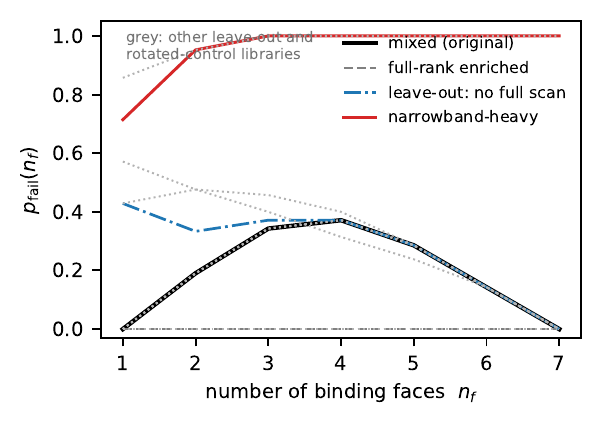}
\caption{Exact-failure rate against contract size across library perturbations. The interior peak of the original library (black) is not reproduced by the narrowband-heavy library (monotone to one) or by the library without a full angular scan, and vanishes entirely once every probe carries $0.1\%$ of a full scan---although near-inestimability does not vanish with it.}
\label{fig:robust}
\end{figure}

\emph{(c) The two budgets diverge as a normal rotates into the weakly measured subspace.} Rotating a face normal from first-harmonic asymmetry toward the high-harmonic waviness index under a displacement-only library at $\rho=0.4175$, the budget to test the face and the budget to certify an action on it separate by three orders of magnitude at intermediate angles (Fig.~\ref{fig:diverge}): at $45^\circ$, testing costs $7.7$ deployments and certifying $20\,428$, and the decoupling ratio $\chi$ of Remark~\ref{rem:kantorovich} peaks at $625$. The peak is at intermediate angles rather than at the extreme: a normal lying wholly in the weak subspace is hard to test as well, so both budgets grow together, while a mixed normal is cheap to test and ruinous to certify. This is the finite-rank shadow of the third regime of Theorem~\ref{thm:separation}, and it is where a margin audit and a permissiveness audit will disagree in practice.

\begin{figure}[t]
\centering
\includegraphics[width=0.62\textwidth]{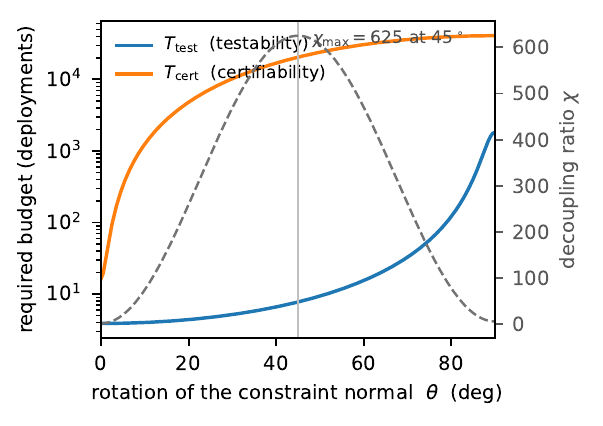}
\caption{Budget to test a face against budget to certify an action on it, as the constraint normal rotates from a well-transmitted mode ($\theta=0$) into the attenuated subspace ($\theta=90^\circ$); slack $0.70$ decision units, $D$-optimal design, displacement-only library. The decoupling ratio $\chi$ (dashed, right axis) peaks at intermediate angles rather than at the extreme: a normal lying wholly in the weak subspace is also hard to test, so both budgets grow together, whereas a mixed normal is cheap to test and ruinous to certify.}
\label{fig:diverge}
\end{figure}

\subsection{Verification of the equivalence conditions}

At the permissiveness optimum of the six-action instance ($T=400$, $\eta=\sqrt{q/T}$), the program attains $\tau^\star=0.3206$ with two active faces---extended service on waviness ($s=0.70$, $r^\star=1.605$) and acceptance without re-shim on clearance at a bolt station ($s=0.60$, $r^\star=1.182$)---and two support probes. Recovering the multipliers gives $\mu=(0.486,0.514)$; the sensitivity $\tr[(M^\star)^{-1}A^\star(M^\star)^{-1}Q(u)]$ equals the threshold $\tr[(M^\star)^{-1}A^\star]=1.3873$ on both support probes to $10^{-8}$ and is strictly below it elsewhere, and the balance-law residual $r^\star_\ell-(s_\ell-\tau^\star)/\eta$ vanishes to machine precision. Fig.~\ref{fig:equiv} shows the sensitivity function against the threshold across the excitation library. The bottleneck certificate names two faces against the Carath\'eodory bound of $r(r+1)/2+1=92$: optimality of the design is witnessed by two constraints out of fifteen action--face pairs. An independent three-dimensional instance with five probes and four faces reproduces the same checks ($\tau^\star=0.1422168$, two active faces, three support probes, $\mu=(0.345,0.655)$, threshold $1.5000882$ matched to $10^{-8}$ on the support and $1.462$, $1.088$ off it).

\begin{figure}[t]
\centering
\includegraphics[width=0.62\textwidth]{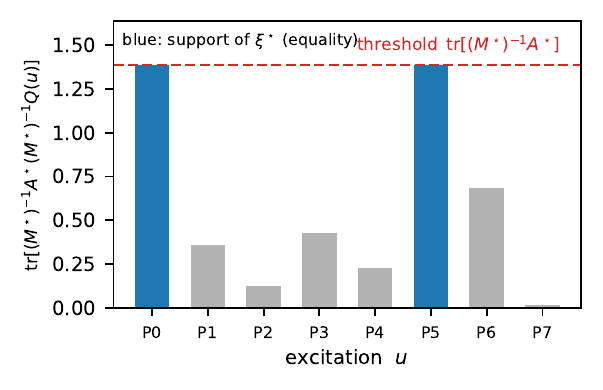}
\caption{Verification of the equivalence condition \eqref{eq:equiv} at the permissiveness optimum: the generalized sensitivity meets the threshold exactly on the support of $\xi^\star$ (to $10^{-8}$) and lies strictly below it on every other excitation.}
\label{fig:equiv}
\end{figure}

\subsection{What the two audits together say}

The margin audit says that for broadband instrumentation a contract-agnostic design is nearly optimal and adaptivity has little to add. The certified-slack audit says that the same conclusion does not transfer to admission: a design that satisfies every margin requirement can leave an action functional inestimable outright, and---in a form that survives every library perturbation we tried, including one that repairs the rank deficiency exactly---can leave it at a certification cost above ten times its testing cost. Whether the failure is exact or merely severe, and how it varies with the number of binding faces, are properties of the incidence between the measured subspace and the contract normals; the mechanism is not. The audits agree when all contract normals sit comfortably inside the measured span and disagree exactly as a normal approaches its boundary---which is what Theorem~\ref{thm:separation} predicts, and the reason a permissiveness audit is not an optional refinement of a margin audit.

\section{Discussion}

\emph{What changes with the object.} The classical design literatures optimize what can be known; this paper optimizes what can be \emph{done}. Theorem~\ref{thm:separation} shows the two can break apart to the point of admission infeasibility while discrimination is perfect, and Remark~\ref{rem:kantorovich} bounds the decoupling at full rank by the condition number. Theorem~\ref{thm:faceequiv} shows that once several action faces compete, the optimal design solves an $L$-optimal problem whose target it generates from its own bottlenecks, with a slack-inverse balance law and a bottleneck certificate of bounded size. Where the object does \emph{not} change the mathematics---single-face design, which is $c$-optimal and goal-oriented design exactly---we say so.

\emph{What the audit adds.} The two audits agree wherever contract normals sit inside the measured span and part exactly at its boundary: a margin-optimal design can give every witness a strictly positive margin and still leave a face non-estimable, at a cost that is a feasibility break rather than a degradation. Across every library we tested the break is common at the level that matters operationally---a certification cost above ten times the testing cost for at least one face in $60\%$ to $100\%$ of three-face contracts---while its dependence on contract size, and even the binary rank criterion itself, are properties of the incidence between probe subspaces and contract normals rather than laws.

\emph{Predictions.} (P1) In a steady-state diffusive network with $W$ ports the accumulated information rank, hence $N^{*}$, is bounded by $W(W-1)$, and excitation rounds beyond it yield nothing; whether the bound is attained is empirical and we do not prejudge it. (P2) The certifiable dimension saturates strictly below the structural bound whenever the spectrum decays. (P3) The fitted slope is $2/\ln(1/\rho)$, with upward deviation for high-mode-amplifying libraries. (P4) The set of directions failing the margin test is not predictable from geometric proximity to the nearest sensing port (correlation weaker than $|r|=0.4$)---the most diagnostic of the four, since its failure would collapse the machinery to a proximity map.

\emph{Limitations.} Local and linear-Gaussian throughout; the mode-level equivalence needs convex whitened mode images and the contract-level statement is a bracket; Theorem~\ref{thm:faceequiv} assumes a finite face set and an attained optimum with $M^\star$ nonsingular on the quotient, and does not address how the active set changes with budget; Theorem~\ref{thm:transfer} is exact only for constraints affine over the confidence region; the audit is on a constructed semi-physical instance rather than a qualified one, its action layer is illustrative, and its contract-size profiles are library dependent by construction; the adaptive design problem is bounded, measured, and left open.

\end{document}